\documentclass{article}
\usepackage{ijcai20}

\usepackage{times}
\usepackage{soul}
\usepackage{url}
\usepackage[hidelinks]{hyperref}
\usepackage[utf8]{inputenc}
\usepackage[small]{caption}
\usepackage{graphicx}
\usepackage{amsmath}
\usepackage{amsthm}
\usepackage{booktabs}
\usepackage{algorithm}
\usepackage{algorithmicx}
\usepackage[toc,page,header]{appendix}
\usepackage{minitoc}

\newtheorem{theorem}{Theorem}

\usepackage{paralist}

\usepackage{enumerate}
\usepackage{commath}
\usepackage{amsmath, amsthm, amssymb, amstext, amsfonts, siunitx, float, graphicx, graphicx, array}
\usepackage{epstopdf}
\usepackage{mathdots}
\usepackage{braket}
\usepackage{float}
\usepackage{listings}
\usepackage{hyperref}
\usepackage{todonotes}
\usepackage[capitalize]{cleveref}

\newtheorem{lemma}{Lemma}
\newtheorem{definition}{Definition}
\usepackage[noend]{algpseudocode}
\newcommand{\nc}[2]{\newcommand{#1}{#2}}
\newcommand{\rnc}[2]{\renewcommand{#1}{#2}}

\newcommand{\undb}[2]{\underbrace{#1}_\text{#2}}
\newcommand{\lrp}[1]{\left(#1\right)}
\newcommand{\lrs}[1]{\left[#1\right]}
\newcommand{\lrc}[1]{\left\{#1\right\}}

\newcommand{\given}{\middle\vert}
\DeclareMathOperator{\E}{\mathbb{E}}
\DeclareMathOperator*{\argmax}{arg\,max}

\nc{\fa}{\forall}\nc{\te}{\exists}\nc{\pa}{\partial}\nc{\inft}{\infty}\nc{\imp}{\implies}\nc{\ds}{\displaystyle}\nc{\spc}{\text{\,\,\,\,}}\nc{\Spc}{\,\,\,\,\,\,\,\,}

\rnc{\l}{\ell }\nc{\R}{\mathbb{R} }\nc{\Ha}{\mathcal{H} }\nc{\La}{\mathcal{L} }\nc{\nb}{\vc{\nabla}}\nc{\tm}{\times}\nc{\Z}{\mathbb{Z}}\nc{\arrow}{\longrightarrow}\rnc{\mod}{\text{mod }}\nc{\kb}{(\SI{1.380E-23}{J K^{-1}})}\nc{\hb}{(\SI{1.055E-34}{J s})}\rnc{\c}{(\SI{3.00E8}{m s^{-1}})}\nc{\sq}{\Box}\rnc{\dag}{\dagger}\nc{\Sol}{\textbf{Solution:}}\nc{\done}{\;\blacksquare}\nc{\ua}{\uparrow}\nc{\da}{\downarrow}\nc{\rank}{\text{rank}}\nc{\Tr}{\text{tr}}\nc{\iidsim}{\stackrel{i.i.d}{\sim}}\rnc{\P}{\mathbb{P}}

\title{Sparse Tree Search Optimality Guarantees in POMDPs with Continuous Observation Spaces}

\author{
Michael H. Lim$^1$
\and
Claire J. Tomlin$^1$\And
Zachary N. Sunberg$^{2}$
\affiliations
$^1$University of California, Berkeley\\
$^2$University of Colorado, Boulder
\emails
michaelhlim@berkeley.edu,
tomlin@eecs.berkeley.edu,
zachary.sunberg@colorado.edu
}

\begin{document}

\doparttoc 
\faketableofcontents 

\maketitle

\begin{abstract}
    Partially observable Markov decision processes (POMDPs) with continuous state and observation spaces have powerful flexibility for representing real-world decision and control problems but are notoriously difficult to solve.
    Recent online sampling-based algorithms that use observation likelihood weighting have shown unprecedented effectiveness in domains with continuous observation spaces.
    However there has been no formal theoretical justification for this technique.
    This work offers such a justification, proving that a simplified algorithm, partially observable weighted sparse sampling (POWSS), will estimate Q-values accurately with high probability and can be made to perform arbitrarily near the optimal solution by increasing computational power.
\end{abstract}
{\let\thefootnote\relax\footnote{{Proceedings of The 29th International Joint Conference on Artificial Intelligence and the 17th Pacific Rim International Conference on Artificial Intelligence (IJCAI-PRICAI 2020). The sole copyright holder is IJCAI (International Joint Conferences on Artificial Intelligence), all rights reserved. Full link to the paper: \href{https://www.ijcai.org/Proceedings/2020/0572.pdf}{https://www.ijcai.org/Proceedings/2020/0572.pdf}. }}}

\section{Introduction}

The partially observable Markov decision process (POMDP) is a flexible mathematical framework for representing sequential decision problems where knowledge of the state is incomplete \cite{kaelbling1998planning,kochenderfer2015decision,bertsekas2005dynamic}.
The POMDP formalism can represent a wide range of real world problems including autonomous driving \cite{sunberg2017value,bai2015intention}, cancer screening \cite{ayer2012mammography}, spoken dialog systems \cite{young2013pomdp}, and others \cite{cassandra1998survey}.
In one of the most successful applications, an approximate POMDP solution is being used in a new aircraft collision avoidance system that will be deployed worldwide~\cite{holand2013optimizing}.

A POMDP is an optimization problem for which the goal is to find a policy that specifies actions that will control the state to maximize the expectation of a reward function.
One of the most popular ways to deal with the challenging computational complexity of finding such a policy~\cite{papadimitriou1987complexity} is to use online tree search algorithms \cite{Silver2010,ye2017despot,sunberg2017value,kurniawati2016online}.
Instead of attempting to find a global policy that specifies actions for every possible outcome of the problem, \emph{online} algorithms look for local approximate solutions as the agent is interacting with the environment.

\begin{figure}[t]
    \centering
    \includegraphics[width=0.45\columnwidth]{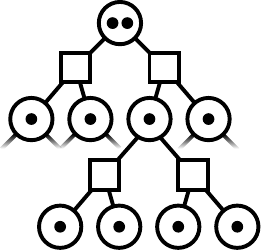}
    \includegraphics[width=0.45\columnwidth]{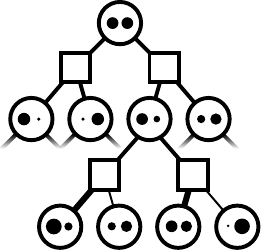}
    \caption{Trees generated from partially observable sparse sampling (POSS) algorithm (left), and partially observable weighted sparse sampling (POWSS) algorithm (right) with depth $D=2$ and width $C=2$, for a continuous-observation POMDP. Nodes below the fading edges are omitted for clarity. Square nodes correspond to actions, filled circles to state particles with size representing weight, and unfilled circles to beliefs.}
    \label{fig:trees}
\end{figure}

Previous online approaches such as ABT \cite{kurniawati2016online}, POMCP \cite{Silver2010}, and DESPOT \cite{ye2017despot} have exhibited good performance in large discrete domains.
However, many real-world domains, notably when a robot interacts with the physical world, have continuous observation spaces, and the algorithms mentioned above will not always converge to an optimal policy in problems with continuous or naively discretized observation spaces \cite{Sunberg2017}.

Two recent approaches, POMCPOW \cite{Sunberg2017} and DESPOT-$\alpha$ \cite{Garg2019}, have employed a weighting scheme inspired by particle filtering to achieve good performance on realistic problems with large or continuous observation spaces. However, there are currently no theoretical guarantees that these algorithms will find optimal solutions in the limit of infinite computational resources.

A convergence proof for these algorithms must have the following two components: (1) A proof that the particle weighting scheme is sound, and (2) a proof that the heuristics used to focus search on important parts of the tree are sound.
This paper tackles the first component by analyzing a new simplified algorithm that expands every node of a sparse search tree.

First, we naively extend the sparse sampling algorithm of Kearns \textit{et al.}~\shortcite{Kearns2002} to the partially observable domain in the spirit of POMCP and explain why this algorithm, known as partially observable sparse sampling (POSS), will converge to a suboptimal solution when the observation space is continuous.
Then, we introduce appropriate weighting that results in the partially observable weighted sparse sampling (POWSS) algorithm.
We prove that the value function estimated by POWSS converges to the optimal value function at a rate of $\mathcal{O}(C^{D}\exp(-t\cdot C))$, where $C$ is the planning width and number of particles, $D$ is the depth, and $t$ is a constant specific to the problem and desired accuracy.
This yields a policy that can be made arbitrarily close to optimal by increasing the computation.

To our knowledge, POWSS is the first algorithm proven to converge to a globally optimal policy for POMDPs with continuous observation spaces without relying on any discretization schemes.\footnote{Edit (06/03/2023): This characterization is not accurate, since there exist other algorithms that previously have shown convergence for continuous observation POMDPs \cite{bai2014integrated}, and those that utilize particle weighting for importance sampling \cite{luo2019importance}. However, we believe it is the first theoretical analysis of online POMDP tree search algorithms that use particle filters weighted using the observation density.}
Since POWSS fully expands all nodes in the sparse tree, it is not computationally efficient and is only practically applicable to toy problems.
However, the convergence guarantees justify the weighting schemes in state-of-the-art efficient algorithms like DESPOT-$\alpha$ and POMCPOW that solve realistic problems by only constructing the most important parts of the search tree.

The remainder of this paper proceeds as follows:
First, \cref{sec:prelim,sec:related} review preliminary definitions and previous work.
Then, \cref{sec:algorithms} presents an overview of the POSS and POWSS algorithms.
Next, \cref{sec:is} contains an importance sampling result used in subsequent sections. 
\Cref{sec:maintheory} contains the main contribution, a proof that POWSS converges to an optimal policy using induction from the leaves to the root to prove that the value function estimate will eventually be accurate with high probability at all nodes. 
Finally, \cref{sec:experiments} empirically shows convergence of POWSS on a modified tiger problem \cite{kaelbling1998planning}. 

\section{Preliminaries} \label{sec:prelim}

\subsubsection{POMDP Formulation}
A POMDP is defined by a 7-tuple $(S,A,O,\mathcal{T},\mathcal{Z},R,\gamma)$: $S$ is the state space, $A$ is the action space, $O$ is the observation space, $\mathcal{T}$ is the transition density $\mathcal{T}(s'|s,a)$, $\mathcal{Z}$ is the observation density $\mathcal{Z}(o|a,s')$, $R$ is the reward function, and $\gamma \in [0,1)$ is the discount factor \cite{kochenderfer2015decision,bertsekas2005dynamic}. 
Since a POMDP agent receives only observations, the agent infers the state by maintaining a belief $b_t$ at each step $t$ and updating it with new action and observation pair $(a_{t+1},o_{t+1})$ via Bayesian updates \cite{kaelbling1998planning}. 
A policy, denoted with $\pi$, maps beliefs $b_t$ generated from histories $h_t = (b_0,a_1,o_1,\cdots,a_t,o_t)$ to actions $a_t$. 
Thus, to maximize the expected cumulative reward in POMDPs, the agent wants to find the optimal policy $\pi^*(b_t)$.

We solve the finite-horizon problem of horizon length $D$. We formulate the state value function $V$ and action value function $Q$ for a given belief state $b$ and policy $\pi$ at step $t$ by Bellman updates for $t \in [0,D-1]$, where $bao$ indicates the belief $b$ updated with $(a,o)$:
\begin{align}
  V_t^\pi(b) &= \E\lrs{\sum_{i=t}^{D-1} \gamma^{i-t} R(s_i, \pi(s_i)) \given b},\; V_D^\pi(b) = 0\\
  Q_t^\pi(b,a) &= \E\lrs{R(s,a) + \gamma V_{t+1}^\pi(bao)| b}
\end{align}
Specifically, the optimal value functions satisfy the following:
\begin{align}
  V_t^*(b) &= \max_{a \in A}Q_t^*(b,a)\\
  \pi_t^*(b) &= \argmax_{a\in A}Q_t^*(b,a)\\
  Q_t^*(b,a) &= \E\lrs{R(s,a) + \gamma V_{t+1}^*(bao)| b}
\end{align}

\paragraph{Generative models.} 
For many problems, the probability densities $\mathcal{T}$ and $\mathcal{Z}$ may be difficult to determine explicitly. Thus, some approaches only require that samples are generated according to the correct probability. In this case, a generative model $G$ implicitly defines $\mathcal{T},\mathcal{Z}$ by generating a new state, $s'$, observation, $o$, and reward, $r$, given the current state $s$ and action $a$.

\paragraph{Probability notation.}
We denote probability measures with calligraphic letters (e.g. $\mathcal{P},\mathcal{Q}$) to avoid confusion with the action value function $Q(b,a)$. Furthermore, for two probability measures $\mathcal{P},\mathcal{Q}$ defined on a $\sigma$-algebra $\mathcal{F}$, we denote $\mathcal{P} \ll \mathcal{Q}$ to state that $\mathcal{P}$ is absolutely continuous with respect to $\mathcal{Q}$; for every measurable set $A,$ $\mathcal{Q}(A) = 0$ implies that $\mathcal{P}(A)=0$. 
Also, we use the abbreviations ``a.s.'' for almost surely, and ``i.i.d.r.v.'' for independent and identically distributed random variables.

\section{Additional Related Work} \label{sec:related}

In addition to the work mentioned in the introduction, there has been much work in similar areas.
There are several online tree search techniques for fully observable Markov decision processes with continuous state spaces, most prominently Sparse-UCT \cite{bjarnason2009lower}, and double progressive widening \cite{couetoux2011double}.

There are also several approaches for solving POMDPs or belief-space MDPs with continuous observation spaces.
For example, Monte Carlo Value Iteration (MCVI) can use a classifier to deal with continuous observation spaces \cite{bai2014integrated}. Others partition the observation space \cite{hoey2005solving} or assume that the most likely observation is always received \cite{platt2010belief}.
Other approaches are based on motion planning \cite{bry2011rapidly,agha2011firm}, locally optimizing pre-computed trajectories \cite{van2012motion}, or optimizing open-loop plans \cite{sunberg2013information}.
McAllester and Singh~\shortcite{mcallester1999approximate} also extend the sparse sampling algorithm of Kearns \textit{et al.}~\shortcite{Kearns2002}, but they use a belief simplification scheme instead of the particle sampling scheme.

\section{Algorithms} \label{sec:algorithms}
We first define the algorithmic elements shared by POSS and POWSS, \textsc{SelectAction} and \textsc{EstimateV}, in \cref{alg:general}.
\textsc{SelectAction} is the entry point of the algorithm, which selects the best action for a belief $b_0$ according to the $Q$-function by recursively calling \textsc{EstimateQ}.
\textsc{EstimateV} is a subroutine that returns the value, $V$, for an estimated belief, by calling \textsc{EstimateQ} for each action and returning the maximum.
We use belief particle set $\bar{b}$ at every step $d$, which contain pairs $(s_i,w_i)$ that correspond to the generated sample and its corresponding weight.
The weight at initial step is uniformly normalized to $1/C$, as the samples are drawn directly from $b_0$.
In \cref{alg:general,alg:poss,alg:powss}, we omit $\gamma,G,C,D$ in the subsequent recursive calls for convenience since they are fixed globally.

\begin{algorithm}[t]
\algrenewcommand\algorithmicprocedure{\textbf{algorithm}}
\caption{Routines common to POWSS and POSS}\label{alg:general}
\textbf{Algorithm:} SelectAction($b_0, \gamma, G, C, D$)\\
\textbf{Input:} Belief state $b_0$, discount $\gamma,$ generative model $G,$ width $C,$ max depth $D$.\\
\textbf{Output:} An action $a^*$.
\begin{algorithmic}[1]
\State From the initial belief distribution $b_0$, sample $C$ particles and store it in $\bar{b}_0$. For POWSS, weights are also initialized with $w_i = 1/C$.
\State For each of the actions $a \in A$, calculate:
\begin{align*}
  \hat{Q}_0^*(\bar{b}_0,a) &= \textsc{EstimateQ}(\bar{b}_0, a, 0)
\end{align*}
\State Return $a^* = \argmax_{a \in A} \hat{Q}_0^*(\bar{b}_0,a)$
\end{algorithmic}

\textbf{Algorithm:} EstimateV($\bar{b}, d$)\\
\textbf{Input:} Belief particles $\bar{b}$, current depth $d$.\\
\textbf{Output:} A scalar $\hat{V}^*_d(\bar{b})$ that is an estimate of $V^*_d(b)$.
\begin{algorithmic}[1]
\State If $d \geq D$ the max depth, then return 0.
\State For each of the actions $a \in A$, calculate:
\begin{align*}
  \hat{Q}^*_d(\bar{b},a) &= \textsc{EstimateQ}(\bar{b}, a, d)
\end{align*}
\State Return $\hat{V}^*_d(\bar{b}) = \max_{a \in A} \hat{Q}^*_d(\bar{b},a)$
\end{algorithmic}
\end{algorithm}

We define \textsc{EstimateQ} functions in Algorithm \ref{alg:poss} for POSS and Algorithm \ref{alg:powss} for POWSS, where both methods perform sampling and recursive calls to \textsc{EstimateV} to estimate the $Q$-function at a given step. 
The crucial difference between these algorithms is shown in \cref{fig:trees}.

POSS naively samples the next $s_i',o_i,r_i$ via the generating function for each state $s_i$ in the belief particle set $\bar{b}$, at a given step $d$. 
Then, for each unique observation $o_j$ generated from the sampling step, POSS inserts the states $s'_i$ into the next-step belief particle set $\overline{bao_j}$ only if the generated observation $o_i$ matches $o_j$.
This behavior is similar to POMCP, DESPOT, or a particle filter that uses rejection and can quickly lead to particle depletion when there are many unique observations.
Finally, POSS returns the naive average of the $Q$-functions calculated via recursive calculation of \textsc{EstimateV} for each of the next-step beliefs.

On the other hand, POWSS uses particle weighting rather than using only unweighted particles with matching observation histories as in POSS. 
POWSS samples the next $s_i',o_i,r_i$ via the generating function for each state-weight pair $(s_i,w_i)$ in the belief particle set $\bar{b}$. 
Now, for each observation $o_j$ generated from the sampling step, POWSS inserts all the states $s'_i$ and the new weights $w_i' = w_i \cdot Z(o_j|a,s_i')$ into the next-step belief particle set $\overline{bao_j}$.
These weights are the adjusted probability of hypothetically obtaining $o_j$ from state $s_i'$.
POWSS then returns the weighted average of the $Q$-functions.

\begin{algorithm}[t]
\algrenewcommand\algorithmicprocedure{\textbf{algorithm}}
\caption{POSS}\label{alg:poss}
\textbf{Algorithm:} EstimateQ($\bar{b}, a, d$)\\
\textbf{Input:} Belief particles $\bar{b}$, action $a$, current depth $d$.\\
\textbf{Output:} A scalar $\hat{Q}^*_d(\bar{b},a)$ that is an estimate of $Q_d^*(b,a)$.
\begin{algorithmic}[1]
\State For each particle $s_i$ in $\bar{b}$, generate $s_i',o_i,r_i = G(s_i,a)$. 
If $i > |\bar{b}|$, use $s_{i\, \mod |\bar{b}|}$. 
\label{line:sample}
\State For each unique observation $o_j$ from previous step, insert all $s_i'$ that satisfy $o_i = o_j$ to a new belief particle set $\overline{bao_j}$.
\State Return
\end{algorithmic}
\begin{align*}
  &\hat{Q}^*_d(\bar{b},a) = \frac{1}{C}\sum_{i = 1}^{C} (r_i+\gamma \cdot\textsc{EstimateV}(\overline{bao_i}, d+1))
\end{align*}
\end{algorithm}

\begin{algorithm}[t]
\algrenewcommand\algorithmicprocedure{\textbf{algorithm}}
\caption{POWSS}\label{alg:powss}
\textbf{Algorithm:} EstimateQ($\bar{b}, a, d$)\\
\textbf{Input:} Belief particles $\bar{b}$, action $a$, current depth $d$.\\
\textbf{Output:} A scalar $\hat{Q}^*_d(\bar{b},a)$ that is an estimate of $Q_d^*(b,a)$.
\begin{algorithmic}[1]
\State For each particle-weight pair $(s_i,w_i)$ in $\bar{b}$, generate $s_i',o_i,r_i$ from $G(s_i,a)$.
\State For each observation $o_j$ from previous step, iterate over $i=\{1,\cdots,C\}$ to insert $(s_i', w_i \cdot Z(o_j|a,s_i'))$ to a new belief particle set $\overline{bao_j}$.
\State Return
\end{algorithmic}
\begin{align*}
  &\hat{Q}^*_d(\bar{b},a) = \frac{\sum_{i = 1}^C w_i(r_i+\gamma\cdot\textsc{EstimateV}(\overline{bao_i}, d+1))}{\sum_{i = 1}^C w_i}
\end{align*}
\end{algorithm}

\section{Importance Sampling} \label{sec:is}

We begin the theoretical portion of this work by stating important properties about self-normalized importance sampling estimators (SN estimators). 
One goal of importance sampling is to estimate an expected value of a function $f(x)$ where $x$ is drawn from distribution $\mathcal{P}$, while the estimator only has access to another distribution $\mathcal{Q}$ along with the importance weights $w_{\mathcal{P}/\mathcal{Q}}(x) \propto \mathcal{P}(x)/\mathcal{Q}(x)$.
This is crucial for POWSS because we wish to estimate the reward for beliefs conditioned on observation sequences, while only being able to generate the marginal distribution of states with correct probability for an action sequence.

We define the following quantities:
\begin{align}
  \tilde{w}_{\mathcal{P}/\mathcal{Q}}(x) &\equiv \frac{w_{\mathcal{P}/\mathcal{Q}}(x)}{\sum_{i=1}^N w_{\mathcal{P}/\mathcal{Q}}(x_i)} \tag{SN Importance Weight}\\
  d_\alpha(\mathcal{P}||\mathcal{Q}) &\equiv \E_{x \sim \mathcal{Q}}[w_{\mathcal{P}/\mathcal{Q}}(x)^\alpha] \tag{R\'enyi Divergence}\\
  \tilde{\mu}_{\mathcal{P}/\mathcal{Q}} &\equiv \sum_{i=1}^N \tilde{w}_{\mathcal{P}/\mathcal{Q}}(x_i) f(x_i) \tag{SN Estimator}
\end{align}

\begin{theorem}[SN $d_{\inft}$-Concentration Bound]\label{thm:sn}
  Let $\mathcal{P}$ and $\mathcal{Q}$ be two probability measures on the measurable space $(\mathcal{X},\mathcal{F})$ with $\mathcal{P} \ll \mathcal{Q}$ and $d_{\inft}(\mathcal{P}||\mathcal{Q}) < +\inft$. Let $x_1,\cdots, x_N$ be i.i.d.r.v. sampled from $\mathcal{Q}$, and $f: \mathcal{X} \to \R$ be a bounded Borel function ($\norm{f}_\inft < + \inft$). Then, for any $\lambda > 0$ and $N$ large enough such that $\lambda > \norm{f}_{\inft}  d_{\inft}(\mathcal{P}||\mathcal{Q})/\sqrt{N}$, the following bound holds with probability at least $1-3\exp(-N\cdot t^2(\lambda,N))$:
  \begin{align} 
    |\E_{x\sim \mathcal{P}}[f(x)] - \tilde{\mu}_{\mathcal{P}/\mathcal{Q}}| &\leq \lambda
  \end{align} 
  where $t(\lambda,N)$ is defined as:
  \begin{align}
    t(\lambda,N) &\equiv \frac{\lambda}{\norm{f}_{\inft}d_{\inft}(\mathcal{P}||\mathcal{Q})}-\frac{1}{\sqrt{N}}
  \end{align}
\end{theorem} 
\cref{thm:sn} builds upon the derivation in Proposition D.3 of Metelli \textit{et al.}~\shortcite{Metelli2018}, which provides a polynomially decaying bound by assuming $d_2$ exists. Here, we compromise by further assuming that the infinite R\'enyi divergence $d_{\inft}$ exists and is bounded to get an exponentially decaying bound: $d_{\inft}(\mathcal{P}||\mathcal{Q}) = \text{ess}\sup_{x \sim \mathcal{Q}} w_{\mathcal{P}/\mathcal{Q}}(x) < +\inft$. The proof of \cref{thm:sn} is given in \cref{app:theorem1}.

This exponential decay is important for the proofs in \cref{sec:powss}. We need to ensure that all nodes of the POWSS tree at all depths $d$ reach convergence. The branching of the tree induces a factor proportional to $N^{D}$. To offset this, we need a probabilistic bound at each depth that decays exponentially with $N$. Intuitive explanation of the $d_{\infty}$ assumption is given at the beginning of \cref{sec:powss}.

\section{Convergence} \label{sec:maintheory}

\subsection{POSS Convergence to QMDP}

We present a short informal argument for the convergence of the $Q$-value estimates of POSS to the QMDP value (\cref{def:qmdp}) in continuous observation spaces.
Sunberg and Kochenderfer~\shortcite{Sunberg2017} provide a formal proof for a similar algorithm.

\begin{definition}[QMDP value] \label{def:qmdp}
    Let $Q_{\textsc{MDP}}(s,a)$ denote the optimal $Q$-function evaluated at state $s$ and action $a$ for the fully observable MDP relaxation of a POMDP.
    Then, the \emph{QMDP value} at belief $b$, $Q_{\textsc{MDP}}(b,a)$, is $\E_{s \sim b}\left[Q_{\textsc{MDP}}(s,a)\right]$.
\end{definition}
Since the observations $o_i$ are drawn from a continuous distribution, the probability of obtaining duplicate $o_i$ values in \textsc{EstimateQ}, line~\ref{line:sample} is 0.
Consequently, when evaluating \textsc{EstimateQ}, all the belief particle sets after the root node only contain a single state particle each (\cref{fig:trees}, left), which means that each belief node is merely an alias for a unique state particle.
Therefore, \textsc{EstimateV} performs a rollout exactly as if the current state became entirely known after taking a single action, identical to the QMDP approximation.
Since QMDP is sometimes suboptimal \cite{kaelbling1998planning}, POSS is suboptimal for some continuous-observation POMDPs.

\subsection{Near-Optimality of POWSS} \label{sec:powss}
On the other hand, we claim that the POWSS algorithm can be made to perform arbitrarily close to the optimal policy by increasing the width $C$. 

In analyzing near-optimality of POWSS, we view POWSS $Q$-function estimates as SN estimators, and we apply the concentration inequality result from \cref{thm:sn} to show that POWSS estimates at every node have small errors with high probability. 
Through the near-optimality of the $Q$-functions, we conclude that the value obtained by employing POWSS policy is also near-optimal with further assumptions on the closed-loop POMDP system.

\subsubsection{Assumptions for Analyzing POWSS}
The following assumptions are needed for the proof:
\begin{compactenum}[(i)]
    \item $S$ and $O$ are continuous spaces, and the action space has a finite number of elements, $|A|<+\inft$. \label{req:space}
    \item For any observation sequence $\{o_n\}_j$, the densities $\mathcal{Z},\mathcal{T},b_0$ are chosen such that the R\'enyi divergence of the target distribution $\mathcal{P}^d$ and sampling distribution $\mathcal{Q}^d$ (\cref{eq:pqdef,eq:pqdef2}) is bounded above by $d_{\inft}^{\max} < +\inft$ a.s. for all $d = 0,\cdots,D-1$: 
      $$d_{\inft}(\mathcal{P}^d||\mathcal{Q}^d) = \text{ess sup}_{x \sim \mathcal{Q}^d}w_{\mathcal{P}^d/\mathcal{Q}^d}(x) \leq d_{\inft}^{\max}$$
      \label{req:Renyi}
    \item The reward function $R$ is Borel and bounded by a finite constant $||R||_{\inft} \leq R_{\max} < +\inft$ a.s., and $V_{\max} \equiv \frac{R_{\max}}{1-\gamma}<+\infty$. \label{req:r}
    \item We can sample from the generating function $G$ and evaluate the observation probability density $\mathcal{Z}$. \label{req:generate}
    \item The POMDP terminates after $D < \infty$ steps. \label{req:finite}
\end{compactenum}
Intuitively, condition (\ref{req:Renyi}) means that the ratio of the conditional observation probability to the marginal observation probability cannot be too high. 
Additionally, our results still hold even when either of $S,O$ are discrete, as long as it doesn't violate condition (\ref{req:Renyi}), by appropriately switching the integrals to Riemann sums.

While we restrict our analysis to the case when $\gamma<1$ for a finite horizon problem, the authors believe that similar results can be derived for either when $\gamma = 1$ or when dealing with infinite horizon problems.

\begin{theorem}[Accuracy of POWSS Q-Value Estimates]\label{thm:powss}
  Suppose conditions (\ref{req:space})-(\ref{req:finite}) are satisfied. Then, for a given $\epsilon >0$, choosing constants $C,\lambda,\delta$ that satisfy:
  \begin{align}
    \lambda &= \epsilon(1-\gamma)^2/5,\; \delta = \lambda/(V_{\max}D(1-\gamma)^2)\\
    \delta &\geq 3|A|(3|A|C)^{D}\exp(-C\cdot t^2_{\max})\\
    t_{\max}&(\lambda,C) = \frac{\lambda}{3V_{\max} d_{\inft}^{\max}} - \frac{1}{\sqrt{C}} > 0
  \end{align}
  The $Q$-function estimates obtained for all depths $d=0,\cdots,D-1$ and all actions $a$ are near-optimal with probability at least $1 - \delta$: 
  \begin{align}
      \left|Q^*_d(b_d,a) - \hat{Q}^*_d(\bar{b}_d,a) \right| &\leq \frac{\lambda}{1-\gamma}
    \end{align}
\end{theorem}
\begin{theorem}[POWSS Policy Convergence] \label{thm:powss-grand}
  In addition to conditions (\ref{req:space})-(\ref{req:finite}), assume that the closed-loop POMDP Bayesian belief update step is exact.
  Then, for any $\epsilon > 0$, we can choose a $C$ such that the value obtained by POWSS is within $\epsilon$ of the optimal value function at $b_0$ a.s.:
  \begin{align}
    V^*(b_0)-V^{\textsc{POWSS}}(b_0) &\leq \epsilon
  \end{align}
\end{theorem}

\Cref{thm:powss,thm:powss-grand} are proven sequentially in the following subsections.
We generally follow the proof strategy of Kearns \textit{et al.}~\shortcite{Kearns2002} but with significant additions to account for the belief-based POMDP calculations rather than state-based MDP calculations. 
We use induction to prove a concentration inequality for the value function at all nodes in the tree, starting at the leaves and proceeding up to the root.

\subsubsection{Value Convergence at Leaf Nodes} \label{sec:leafnode}
First, we reason about the convergence at nodes at depth $D-1$ (leaf nodes). 
In the subsequent analysis, we abbreviate some terms of interest with the following notation:
\begin{align}
  \mathcal{T}_{1:d}^i &\equiv \prod_{n=1}^d \mathcal{T}(s_{n,i}|s_{n-1,i},a_n) \\
  \mathcal{Z}_{1:d}^{i,j} &\equiv \prod_{n=1}^d \mathcal{Z}(o_{n,j}|a_{n},s_{n,i})\notag
\end{align}
Here $d$ denotes the depth, $i$ denotes the index of the state sample, and $j$ denotes the index of the observation sample.
Absence of indices $i,j$ means that $\{s_n\}$ and/or $\{o_n\}$ appear as regular variables.
Intuitively, $\mathcal{T}_{1:d}^i$ is the transition density of state sequence $i$ from the root node to depth $d$, and $\mathcal{Z}_{1:d}^{i,j}$ is the conditional density of observation sequence $j$ given state sequence $i$ from the root node to depth $d$.
Additionally, $b_d^i$ denotes $b_d(s_{d,i})$, $r_{d,i}$ the reward $R(s_{d,i},a_{d})$, and $w_{d,i}$ the weight of $s_{d,i}$.

Since the problem ends after $D$ steps, the $Q$-function for nodes at depth $D-1$ is simply the expectation of final reward and the POWSS estimate has the following form:
\begin{align}
  Q^*_{D-1}(b_{D-1},a) &= \int_S R(s_{D-1},a) b_{D-1} ds_{D-1}\\
  \hat{Q}^*_{D-1}(\bar{b}_{D-1},a) &= \frac{\sum_{i = 1}^C w_{D-1,i}r_{D-1,i}}{\sum_{i = 1}^C w_{D-1,i}}
\end{align}

\begin{lemma}[SN Estimator Leaf Node Convergence]\label{lemma:leaf}
  $\hat{Q}^*_{D-1}(\bar{b}_{D-1},a)$ is an SN estimator of $Q^*_{D-1}(b_{D-1},a)$, and the following leaf-node concentration bound holds with probability at least $1-3\exp(-C\cdot t_{\max}^2(\lambda,C))$,
  \begin{align}
    |Q^*_{D-1}(b_{D-1},a) - \hat{Q}^*_{D-1}(\bar{b}_{D-1},a) | \leq \lambda
  \end{align}
\end{lemma}

\begin{proof} First, we show that $\hat{Q}^*_{D-1}(\bar{b}_{D-1},a)$ is an SN estimator of $Q^*_{D-1}(b_{D-1},a)$. By following the recursive belief update, the belief term can be fully expanded:
\begin{align}
  b_{D-1}(s_{D-1}) &= \frac{\int_{S^{D-1}} (\mathcal{Z}_{1:D-1})(\mathcal{T}_{1:D-1}) b_{0} ds_{0:D-2}}{\int_{S^{D}} (\mathcal{Z}_{1:D-1})(\mathcal{T}_{1:D-1}) b_{0} ds_{0:D-1}}
\end{align} 
Then, $Q^*_{D-1}(b_{D-1},a)$ is equal to the following:
\begin{align}
  &Q^*_{D-1}(b_{D-1},a) = \int_S R(s_{D-1},a) b_{D-1} ds_{D-1}\\
  &= \frac{\int_{S^{D}} R(s_{D-1},a) (\mathcal{Z}_{1:D-1})(\mathcal{T}_{1:D-1}) b_{0} ds_{0:D-1}}{\int_{S^{D}} (\mathcal{Z}_{1:D-1})(\mathcal{T}_{1:D-1}) b_{0} ds_{0:D-1}}
\end{align}
We approximate the $Q^*$ function with importance sampling by utilizing problem requirement (\ref{req:generate}), where the target density is $b_{D-1}$. 
First, we sample the sequences $\{s_n\}_i$ according to the joint probability $(\mathcal{T}_{1:D-1}) b_{0}$.
Afterwards, we weight the sequences by the corresponding observation density $\mathcal{Z}_{1:D-1}$, obtained from the generated observation sequences $\{o_n\}_j$.
For now, we assume the observation sequences $\{o_n\}_j$ are fixed.

Applying the importance sampling formalism to our system for all depths $d=0,\cdots,D-1$, $\mathcal{P}^d$ is the normalized measure incorporating the probability of observation sequence $j$ on top of the state sequence $i$ until the node at depth $d$, and $\mathcal{Q}^d$ is the measure of the state sequence. We can think of $\mathcal{P}^d$ being indexed by the observation sequence $\{o_n\}_j$.
\begin{align}\label{eq:pqdef}
  &\mathcal{P}^d = \mathcal{P}_{\{o_n\}_j}^d(\{s_n\}_i) = \frac{(\mathcal{Z}_{1:d}^{i,j}) (\mathcal{T}_{1:d}^i) b_{0}^i}{\int_{S^{d+1}} (\mathcal{Z}_{1:d}^{j}) (\mathcal{T}_{1:d}) b_{0} ds_{0:d}}\\
  &\mathcal{Q}^d = \mathcal{Q}^d(\{s_n\}_i) = (\mathcal{T}_{1:d}^i) b_{0}^i \label{eq:pqdef2}\\
  &w_{\mathcal{P}^d/\mathcal{Q}^d}(\{s_n\}_i) = \frac{(\mathcal{Z}_{1:d}^{i,j}) }{\int_{S^{d+1}} (\mathcal{Z}_{1:d}^{j}) (\mathcal{T}_{1:d}) b_{0} ds_{0:d}} 
\end{align}

The weighing step is done by updating the self-normalized weights given in POWSS algorithm. 
We define $w_{d,i}$ and $r_{d,i}$ as the weights and rewards obtained at step $d$ for state sequence $i$ from POWSS simulation. With our recursive definition of the empirical weights, we obtain the full weight of each state sequence $i$ for a fixed observation sequence $j$:
\begin{align}
  w_{D-1,i} &= w_{D-2,i}\cdot Z(o_{D-1,j}|a_{D-1},s_{D-1,i}) \\
  &\propto \mathcal{Z}_{1:D-1}^{i,j}
\end{align}

Realizing that the marginal observation probability is independent of indexing by $i$, we show that $\hat{Q}^*_{D-1}(\bar{b}_{D-1},a)$ is an SN estimator of $Q^*_{D-1}(b_{D-1},a)$:
\begin{align}
  &\hat{Q}^*_{D-1}(\bar{b}_{D-1},a) = \frac{\sum_{i = 1}^C (\mathcal{Z}_{1:D-1}^{i,j}) R(s_{D-1,i},a)}{\sum_{i = 1}^C (\mathcal{Z}_{1:D-1}^{i,j}) }\\
  &= \frac{\sum_{i = 1}^C \frac{(\mathcal{Z}_{1:D-1}^{i,j})}{\int_{S^{D}} (\mathcal{Z}_{1:D-1}^{j}) (\mathcal{T}_{1:D-1}) b_{0} ds_{0:D-1}} R(s_{D-1,i},a)}{\sum_{i = 1}^C \frac{(\mathcal{Z}_{1:D-1}^{i,j})}{\int_{S^{D}} (\mathcal{Z}_{1:D-1}^{j}) (\mathcal{T}_{1:D-1}) b_{0} ds_{0:D-1}} }\\
  &= \frac{\sum_{i = 1}^C w_{\mathcal{P}^{D-1}/\mathcal{Q}^{D-1}}(\{s_n\}_i) R(s_{D-1,i},a)}{\sum_{i = 1}^C w_{\mathcal{P}^{D-1}/\mathcal{Q}^{D-1}}(\{s_n\}_i) }\\
  &= \sum_{i = 1}^C \tilde{w}_{\mathcal{P}^{D-1}/\mathcal{Q}^{D-1}}(\{s_n\}_i) R(s_{D-1,i},a)
\end{align}
Since $\{s_n\}_1,\cdots \{s_n\}_C$ are i.i.d.r.v. sequences of depth $D$ sampled from $\mathcal{Q}^{D-1}$, and $R$ is a bounded function from problem requirement (\ref{req:r}), we can apply the SN concentration bound in Theorem $\ref{thm:sn}$ to prove Lemma $\ref{lemma:leaf}$. Detailed finishing steps of the proof are given in \cref{app:lemma1}.
\end{proof}

\subsubsection{Induction from Leaf to Root Nodes}
Now, we want to show that nodes at all depths have convergence guarantees via induction.
\begin{lemma}[SN Estimator Step-by-Step Convergence]\label{lemma:step}
  $\hat{Q}^*_d(\bar{b}_d,a)$ is an SN estimator of $Q^*_d(b_d,a)$, and for all $d = 0,\cdots,D-1$ and $a$, the following holds with probability at least $1 - 3|A|(3|A|C)^{D}\exp(-C\cdot t^2_{\max})$:
  \begin{align} 
    &|Q^*_d(b_d,a) - \hat{Q}^*_d(\bar{b}_d,a)| \leq \alpha_d \label{eq:dcondition}\\
    &\alpha_{d} \equiv \lambda + \gamma\alpha_{d+1};\; \alpha_{D-1} = \lambda
  \end{align}
\end{lemma}

\begin{proof} First, we set $C$ such that $C > (3V_{\max} d_{\inft}^{\max}/\lambda)^2$ to satisfy $t_{\max}(\lambda,C)>0$, which ensures that the SN concentration inequality holds with probability $1-3\exp(-C\cdot t_{\max}^2(\lambda,C))$ at any given step $d$ and action $a$.
Furthermore, we multiply the worst-case union bound factor $(3|A|C)^{D}$, since we want the function estimates to be within their respective concentration bounds for all the actions $|A|$ and child nodes $C$ at each step $d=0,\cdots,D-1$, for the 3 times we use SN concentration bound in the induction step. We once again multiply the final $\delta$ by $|A|$ to account for the root node $Q$-value estimates also satisfying their respective concentration bounds for all actions. 

Following our definition of \textsc{EstimateQ}, the value function estimates at step $d$ are given as the following:
\begin{align}
\hat{V}^*_d(\bar{b}_d) &= \max_{a\in A}\hat{Q}^*_d(\bar{b}_d,a)\\
  \hat{Q}^*_d(\bar{b}_d,a) &= \frac{\sum_{i = 1}^C w_{d,i}\lrp{r_{d,i} +  \gamma\hat{V}^*_{d+1}(\overline{b_dao_i})}}{\sum_{i = 1}^C w_{d,i}}
\end{align}
The base case $d=D-1$ holds by Lemma \ref{lemma:leaf}.
Then for the inductive step, we assume \cref{eq:dcondition} holds for all actions at step $d+1$. 
Using the triangle inequality for step $d$, we split the difference into two terms, the reward estimation error (A) and the next-step value estimation error (B):
\begin{align}
  &|Q^*_d(b_d,a) - \hat{Q}^*_d(\bar{b}_d,a) | \\
  \begin{split}
    &\leq \undb{\left| \E[R(s_d,a)|b_d] - \frac{\sum_{i = 1}^C w_{d,i}r_{d,i}}{\sum_{i = 1}^C w_{d,i}}\right|}{(A)} \\
    &+ \gamma \undb{\left| \E[V_{d+1}^*(bao)|b_d] - \frac{\sum_{i = 1}^C w_{d,i}\hat{V}^*_{d+1}(\overline{b_dao_i})}{\sum_{i = 1}^C w_{d,i}} \right|}{(B)}
  \end{split}\notag
\end{align}
Each of the error terms are bound by $(A) \leq \frac{R_{\max}}{3V_{\max}}\lambda$ and $(B)\leq \frac{1}{3}\lambda + \frac{2}{3\gamma}\lambda + \alpha_{d+1}$. 
We provide a detailed justification of these bounds in \cref{app:lemma2}, which uses the SN concentration bound 3 times. 
Combining (A) and (B), we prove the inductive hypothesis:
\begin{align}
   |Q^*_d(b_d,a) - \hat{Q}^*_d(\bar{b}_d,a) | &\leq \frac{R_{\max}}{3V_{\max}}\lambda + \gamma [\frac{1}{3}\lambda + \frac{2}{3\gamma}\lambda + \alpha_{d+1}]\notag\\
   &\leq  \lambda + \gamma \alpha_{d+1} = \alpha_d \label{eq:stepbystep}
\end{align}
Therefore, \cref{eq:dcondition} holds for all $d = 0,\cdots,D-1$ with probability at least $1 - 3|A|(3|A|C)^{D}\exp(-C\cdot t^2_{\max})$. \end{proof}

\begin{proof} (Theorem \ref{thm:powss}) First, we choose constants $C,\lambda,\delta$ and densities $\mathcal{Z},\mathcal{T},b_0$ that satisfy the conditions in \cref{thm:powss}. Since $\alpha_d \leq \alpha_0$, the following holds for all $d=0,\cdots, D-1$ with probability at least $1-\delta$ through \cref{lemma:leaf,lemma:step}:
\begin{align}
  |Q^*_d(b_d,a) - \hat{Q}^*_d(\bar{b}_d,a) | & \leq \alpha_0  =  \sum_{d=0}^{D-1} \gamma^{d} \lambda \leq \frac{\lambda}{1-\gamma}
\end{align}
Note that the convergence rate $\delta$ is $\mathcal{O}(C^{D}\exp(-t C))$, where $t = (\lambda/(3V_{\max}d_{\inft}^{\max}))^2$.
\end{proof}

\subsubsection{Near-Optimal Policy Performance}
We have proven in the previous subsection that the planning step results in a near-optimal $Q$-value for a given belief.
Assuming further that we have a perfect Bayesian belief update in the outer observe-plan-act loop, we prove \cref{thm:powss-grand}, which states that the closed-loop POMDP policy generated by POWSS at each planning step results in a near-optimal policy. 
The proof given in \cref{app:theorem3} combines \cref{thm:powss} with results from Kearns \textit{et al.}~\shortcite{Kearns2002}; Singh and Yee~\shortcite{Singh1994}:
\begin{align}\label{eq:powss-grand}
  V^*(b_0)-V^{\textsc{POWSS}}(b_0) &\leq \epsilon
\end{align}

\section{Experiments} \label{sec:experiments}

The simple numerical experiments in this section confirm the theoretical results of \cref{sec:maintheory}. Specifically, they show that the value function estimates of POSS converge to the QMDP approximation and the value function estimates of POWSS converge to the optimal value function for a toy problem.

\subsection{Continuous Observation Tiger Problem}
We consider a simple modification of the classic tiger problem \cite{kaelbling1998planning} that we refer to as the continuous observation tiger (CO-tiger) problem. 
In the CO-tiger problem, the agent is presented with two doors, left (L) and right (R).
One door has a tiger behind it ($S = \{\texttt{Tiger L}, \texttt{Tiger R}\}$).
In the classic problem, the agent can either open one of the doors or listen, and the CO-tiger problem has an additional wait action to illustrate the suboptimality of QMDP estimates ($A = \{\texttt{Open L}, \texttt{Open R}, \texttt{Wait}, \texttt{Listen}\}$). 
If the agent opens a door, the problem terminates immediately; If the tiger is behind that door, a penalty of -10 is received, but if not, a reward of 10 is given.
Waiting has a penalty of -1 and listening has a penalty of -2.
If the agent waits or listens, a noisy continuous observation between 0 and 1 is received ($O = [0,1]$).
In the wait case, this observation is uniformly distributed, independent of the tiger's position, yielding no information.
In the listen case, the observation distribution is piecewise uniform.
An observation in $[0, 0.5]$ corresponds to a tiger behind the left door and $(0.5, 1]$ the right door.
Listening yields an observation in the correct range 85\% of the time.
The discount is $0.95$, and the terminal depth is $3$.

The optimal solution to this problem may be found by simply discretizing the observation space so that any continuous observation in $[0, 0.5]$ is treated as a \texttt{Tiger L} observation, and any continuous observation in $(0.5, 1]$ is treated as a \texttt{Tiger R} observation.
This fully discrete version of the problem may be easily solved by a classical solution method such as the incremental pruning method of Cassandra \textit{et al.}~\shortcite{cassandra1997incremental}.
Given an evenly-distributed initial belief, the optimal action is \texttt{Listen} with a value of 4.65, and the \texttt{Wait} action has a value of 3.42. 
The QMDP estimate for \texttt{Wait} is 8.5 and for \texttt{Listen} is 7.5.

While the CO-tiger problem is too small to be of practical significance, it serves as an empirical demonstration that POWSS converges toward the optimal value estimates and that POSS converges toward the QMDP estimates. 
In fact, the QMDP estimates generated by POSS are suboptimal in this example and lead to picking the suboptimal \texttt{Wait} action.
Both POWSS and POSS were implemented using the POMDPs.jl framework,
\cite{egorov2017pomdps}
and open-source code can be found at 
\url{https://github.com/JuliaPOMDP/SparseSampling.jl}.

\subsection{Results}
The results plotted in \cref{fig:convergence} show the $Q$-value estimates of POWSS converging toward the optimal $Q$-values as the width $C$ is increased.
Each data point represents the mean $Q$-value from 200 runs of the algorithm from a uniformly-distributed belief, with the standard deviation plotted as a ribbon.
The estimates for POSS have no uncertainty bounds since the estimates in this problem are the same for all $C$.

With $C=1$, POWSS suffers from particle depletion and, because of the particular structure of this problem, finds the QMDP $Q$-values.
As $C$ increases, one can observe that both bias and variance in the $Q$-value estimates significantly decrease in agreement with our theoretical results, while POSS continues to yield incorrect estimates.

Some estimates by POMCPOW are also included.
These are not directly comparable since POMCPOW is parameterized differently.
For these tests, the double progressive widening parameters $k_o=C$, $\alpha_o=0$ were used to limit the tree width, with $n=C^3$ iterations to keep the particle density high in wider trees (see Sunberg and Kochenderfer~\shortcite{Sunberg2017} for parameter definitions).
POMCPOW's value estimates are strongly biased downwards by exploration actions, but the estimated value for \texttt{Listen} action is much higher than the estimated value for the \texttt{Wait} action, which is too low to appear on the plot.
Thus the correct action will usually still be chosen.
At $C=41$, POMCPOW is about an order of magnitude faster than POWSS.

\begin{figure}[t]
    \centering
    \includegraphics[width=\columnwidth]{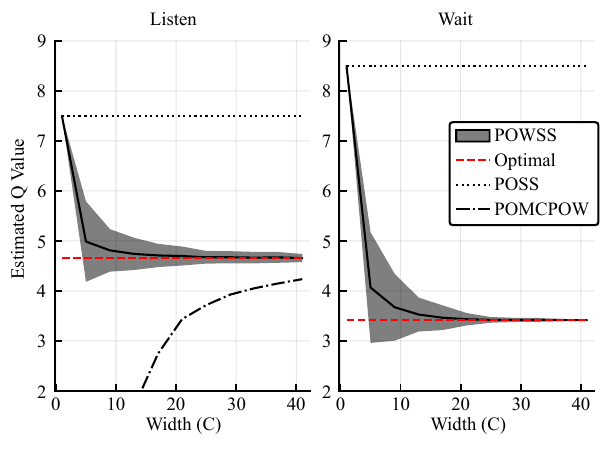}
    \caption{Numerical convergence of $Q$-value estimates for POSS, POWSS, and POMCPOW in the CO-tiger problem. Ribbons indicate standard deviation.}
    \label{fig:convergence}
\end{figure}

\section{Conclusion}

This work has proposed two new POMDP algorithms and analyzed their convergence in POMDPs with continuous observation spaces.
Though these algorithms are not computationally efficient and thus not suitable for realistic problems, this work lays the foundation for analysis of more complex algorithms, rigorously justifying the observation likelihood weighting used in POWSS, POMCPOW, and DESPOT-$\alpha$.

There is a great deal of future work to be done along this path.
Most importantly, the theory presented in this work should be extended to more computationally efficient and hence practical algorithms.
Before extending to POMCPOW and DESPOT-$\alpha$, it may be beneficial to apply these techniques to an algorithm that is less conceptually complex, such as a modification of Sparse-UCT \cite{bjarnason2009lower} extended to partially observable domains.
Such an algorithm could enjoy strong theoretic guarantees, ease of implementation, and good performance on large problems.

Moreover, the proof techniques in this work may yield insight into which problems are difficult for sparse tree search techniques.
For example, the R\'enyi divergence between the marginal and conditional state distributions (assumption (\ref{req:Renyi})) may be a difficulty indicator for likelihood-weighted sparse tree solvers, similar to the covering number of the optimal reachable belief space for point-based solvers \cite{lee2008makes}.

\section*{Acknowledgements}
This material is based upon work supported by a DARPA Assured Autonomy Grant, the SRC CONIX program, NSF CPS Frontiers, the ONR Embedded Humans MURI, and the National Science Foundation Graduate Research Fellowship Program under Grant No. DGE 1752814. 
Any opinions, findings, and conclusions or recommendations expressed in this material are those of the authors and do not necessarily reflect the views of any aforementioned organizations.
The authors also thank Lasse Peters for valuable manuscript feedback.

\bibliographystyle{named}
\bibliography{POWSS-IJCAI}
\newpage

\onecolumn
\appendix
\setcounter{theorem}{0}
\setcounter{lemma}{0}
\setcounter{equation}{0}
\addcontentsline{toc}{section}{Appendix} 
\part{Appendix} 
\parttoc 
\setcounter{page}{1}

\section{Proof of Theorem 1}\label{app:theorem1}
\begin{theorem}[SN $d_{\inft}$-Concentration Bound]
  Let $\mathcal{P}$ and $\mathcal{Q}$ be two probability measures on the measurable space $(\mathcal{X},\mathcal{F})$ with $\mathcal{P} \ll \mathcal{Q}$ and $d_{\inft}(\mathcal{P}||\mathcal{Q}) < +\inft$. Let $x_1,\cdots, x_N$ be i.i.d.r.v. sampled from $\mathcal{Q}$, and $f: \mathcal{X} \to \R$ be a bounded Borel function ($\norm{f}_\inft < + \inft$). Then, for any $\lambda > 0$ and $N$ large enough such that $\lambda > \norm{f}_{\inft}  d_{\inft}(\mathcal{P}||\mathcal{Q})/\sqrt{N}$, the following bound holds with probability at least $1-3\exp(-N\cdot t^2(\lambda,N))$:
  \begin{align} 
    |\E_{x\sim \mathcal{P}}[f(x)] - \tilde{\mu}_{\mathcal{P}/\mathcal{Q}}| &\leq \lambda
  \end{align} 
  where $t(\lambda,N)$ is defined as:
  \begin{align}
    t(\lambda,N) &\equiv \frac{\lambda}{\norm{f}_{\inft}d_{\inft}(\mathcal{P}||\mathcal{Q})}-\frac{1}{\sqrt{N}}
  \end{align}
\end{theorem} 

\begin{proof} This proof follows similar proof steps as in Metelli \textit{et al.}~\shortcite{Metelli2018}. Since we have upper bounds on the infinite R\'enyi divergence $d_{\inft}(\mathcal{P}||\mathcal{Q})$, we can start from the Hoeffding's inequality for bounded random variables applied to the regular IS estimator $\hat{\mu}_{\mathcal{P}/\mathcal{Q}} = \frac{1}{N}\sum_{i=1}^N w_{\mathcal{P}/\mathcal{Q}}(x_i)f(x_i)$, which is unbiased. While applying the Hoeffding's inequality, we can view importance sampling on $f(x)$ weighted by $w_{\mathcal{P}/\mathcal{Q}}(x)$ as Monte Carlo sampling on $g(x) = w_{\mathcal{P}/\mathcal{Q}}(x)f(x)$, which is a function bounded by $\norm{g}_\inft = d_{\inft}(\mathcal{P}||\mathcal{Q})\norm{f}_{\inft}$:
\begin{align}
  \P\lrp{\hat{\mu}_{\mathcal{P}/\mathcal{Q}} - \E_{x \sim P}[f(x)] \geq \lambda} &= \P\lrp{\hat{\mu}_{\mathcal{P}/\mathcal{Q}} - \E_{x \sim Q}[\hat{\mu}_{\mathcal{P}/\mathcal{Q}}(x)f(x)] \geq \lambda}\\
  &\leq \exp\lrp{-\frac{2N^2\lambda^2}{\sum_{i=1}^N 2(d_{\inft}(\mathcal{P}||\mathcal{Q})\norm{f}_{\inft})^2}}\\
  &\leq \exp\lrp{-\frac{N\lambda^2}{d_{\inft}^2(\mathcal{P}||\mathcal{Q})\norm{f}_{\inft}^2}} \equiv \delta\\
  \P\lrp{|\hat{\mu}_{\mathcal{P}/\mathcal{Q}} - \E_{x \sim P}[f(x)]| \geq \lambda} & \leq 2\exp\lrp{-\frac{N\lambda^2}{d_{\inft}^2(\mathcal{P}||\mathcal{Q})\norm{f}_{\inft}^2}} = 2\delta
\end{align}
We prove a similar bound for the SN estimator $\tilde{\mu}_{\mathcal{P}/\mathcal{Q}} = \sum_{i=1}^N \tilde{w}_{\mathcal{P}/\mathcal{Q}}(x_i) f(x_i)$, which is a biased estimator. However, we need to take a step further and analyze the absolute difference, requiring us to split the difference up into two terms:
\begin{align}
  \P(|&\E_{x\sim \mathcal{P}}[f(x)] - \tilde{\mu}_{\mathcal{P}/\mathcal{Q}}|\geq \lambda) \\
  &\leq \P(\tilde{\mu}_{\mathcal{P}/\mathcal{Q}} - \E_{x\sim \mathcal{P}}[f(x)] \geq \lambda) + \P(\E_{x\sim \mathcal{P}}[f(x)] - \tilde{\mu}_{\mathcal{P}/\mathcal{Q}}\geq \lambda)\\
  &\leq \P(\tilde{\mu}_{\mathcal{P}/\mathcal{Q}} - \E_{x\sim \mathcal{Q}}[\tilde{\mu}_{\mathcal{P}/\mathcal{Q}}]\geq \tilde{\lambda}) + \P(\E_{x\sim \mathcal{Q}}[\tilde{\mu}_{\mathcal{P}/\mathcal{Q}}] - \tilde{\mu}_{\mathcal{P}/\mathcal{Q}}\geq \tilde{\lambda})\\
  &\leq \tilde{\delta} + \P(\E_{x\sim \mathcal{P}}[f(x)] - \tilde{\mu}_{\mathcal{P}/\mathcal{Q}}\geq \lambda)
\end{align}
The first term is bounded by $\tilde{\delta}$ from the above bound and recasting $\lambda$ to $\tilde{\lambda}$ to account for the bias of the SN estimator:
\begin{align}
  \tilde{\lambda} &= \lambda - \left|\E_{x\sim \mathcal{P}}[f(x)] - \E_{x\sim \mathcal{Q}}[\tilde{\mu}_{\mathcal{P}/\mathcal{Q}}]\right|\\
  \tilde{\delta} &= \exp\lrp{-\frac{N\tilde{\lambda}^2}{d_{\inft}^2(\mathcal{P}||\mathcal{Q})\norm{f}_{\inft}^2}}
\end{align}
Note that the bias term in the SN estimator is bounded by following through Cauchy-Schwarz inequality, closely following steps from Metelli \textit{et al.}~\shortcite{Metelli2018}:
\begin{align}
  |\E_{x\sim \mathcal{P}}&[f(x)] - \E_{x\sim \mathcal{Q}}[\tilde{\mu}_{\mathcal{P}/\mathcal{Q}}]| = \left|\E_{x\sim \mathcal{Q}}[\hat{\mu}_{\mathcal{P}/\mathcal{Q}}-\tilde{\mu}_{\mathcal{P}/\mathcal{Q}}] \right| \leq \E_{x\sim \mathcal{Q}}[|\hat{\mu}_{\mathcal{P}/\mathcal{Q}}-\tilde{\mu}_{\mathcal{P}/\mathcal{Q}}|]\\
  &\leq \E_{x\sim \mathcal{Q}}\left| \frac{\sum_{i=1}^N w_{\mathcal{P}/\mathcal{Q}}(x_i) f(x_i)}{\sum_{i=1}^N w_{\mathcal{P}/\mathcal{Q}}(x_i) }-\frac{1}{N}\sum_{i=1}^N w_{\mathcal{P}/\mathcal{Q}}(x_i) f(x_i) \right|\\
  &= \E_{x\sim \mathcal{Q}}\lrs{\left| \frac{\sum_{i=1}^N w_{\mathcal{P}/\mathcal{Q}}(x_i) f(x_i)}{\sum_{i=1}^N w_{\mathcal{P}/\mathcal{Q}}(x_i) } \right| \left| 1-\frac{\sum_{i=1}^N w_{\mathcal{P}/\mathcal{Q}}(x_i)}{N} \right|}\\
  &\leq \E_{x\sim \mathcal{Q}}\lrs{ \lrp{\frac{\sum_{i=1}^N w_{\mathcal{P}/\mathcal{Q}}(x_i) f(x_i)}{\sum_{i=1}^N w_{\mathcal{P}/\mathcal{Q}}(x_i) } }^2}^{1/2} \E_{x\sim \mathcal{Q}}\lrs{\lrp{1-\frac{\sum_{i=1}^N w_{\mathcal{P}/\mathcal{Q}}(x_i)}{N}}^2 }^{1/2}\\
  &\leq \norm{f}_{\inft}  \sqrt{\frac{d_{2}(\mathcal{P}||\mathcal{Q}) - 1}{N}} \leq \norm{f}_{\inft}  \frac{d_{\inft}(\mathcal{P}||\mathcal{Q})}{\sqrt{N}}
\end{align}
In the last step, the first term is bounded by $\norm{f}_{\inft}$ as the function is bounded, and the second term is bounded by the fact that we can bound the square root of variance with the supremum squared, where we square it for the convenience of the definition of $t(\lambda,N)$ later on such that the $1/\sqrt{N}$ factor is nicely separated. We assume that $N$ is chosen large enough that $\lambda > \norm{f}_{\inft}  d_{\inft}(\mathcal{P}||\mathcal{Q})/\sqrt{N}.$ Using this, we bound the $\tilde{\delta}$ term:
\begin{align}
  \tilde{\delta}&\leq \exp\lrp{-\frac{N(\lambda - \norm{f}_{\inft}  d_{\inft}(\mathcal{P}||\mathcal{Q})/\sqrt{N})^2}{d_{\inft}^2(\mathcal{P}||\mathcal{Q})\norm{f}_{\inft}^2}}\\
  &= \exp\lrp{-N\lrp{\frac{\lambda - \norm{f}_{\inft}  d_{\inft}(\mathcal{P}||\mathcal{Q})/\sqrt{N}}{\norm{f}_{\inft}d_{\inft}(\mathcal{P}||\mathcal{Q})}}^2 }\\
  &\equiv \exp\lrp{-N\cdot t^2(\lambda,N) }
\end{align}
Here, we define $t(\lambda,N) \equiv \frac{\lambda}{\norm{f}_{\inft}d_{\inft}(\mathcal{P}||\mathcal{Q})} - \frac{1}{\sqrt{N}}$, which satisfies $0 < t(\lambda,N) \leq \frac{\lambda}{\norm{f}_{\inft}d_{\inft}(\mathcal{P}||\mathcal{Q})}$. The second term can be bounded similarly by rebounding the bias term with $\tilde{\lambda}$, using symmetry and Hoeffding's inequality:
\begin{align}
  \P(\E_{x\sim \mathcal{P}}[f(x)] - \tilde{\mu}_{\mathcal{P}/\mathcal{Q}}\geq \lambda) &\leq \P(\E_{x\sim \mathcal{Q}}[\tilde{\mu}_{\mathcal{P}/\mathcal{Q}}] - \tilde{\mu}_{\mathcal{P}/\mathcal{Q}}\geq \tilde{\lambda})\\
  &\leq \P(|\E_{x\sim \mathcal{Q}}[\tilde{\mu}_{\mathcal{P}/\mathcal{Q}}] - \tilde{\mu}_{\mathcal{P}/\mathcal{Q}}|\geq \tilde{\lambda}) \leq 2\tilde{\delta}
\end{align}
Thus, we obtain the following bound:
\begin{align}
  \P(|\E_{x\sim \mathcal{P}}[f(x)] - \tilde{\mu}_{\mathcal{P}/\mathcal{Q}}|\geq \lambda) &\leq 3\exp(-N\cdot t^2(\lambda,N))
\end{align}
\end{proof}

\clearpage
\section{Proof of Lemma 1 (Continued)}\label{app:lemma1}
In the main paper, we show that $\hat{Q}^*_{D-1}(\bar{b}_{D-1},a)$ is an SN estimator of $Q^*_{D-1}(b_{D-1},a)$. We apply the concentration inequality proven in \cref{thm:sn} to finish the proof of \cref{lemma:leaf}.
\begin{lemma}[SN Estimator Leaf Node Convergence]
  $\hat{Q}^*_{D-1}(\bar{b}_{D-1},a)$ is an SN estimator of $Q^*_{D-1}(b_{D-1},a)$, and the following leaf-node concentration bound holds with probability at least $1-3\exp(-C\cdot t_{\max}^2(\lambda,C))$,
  \begin{align}
    |Q^*_{D-1}(b_{D-1},a) - \hat{Q}^*_{D-1}(\bar{b}_{D-1},a) | \leq \lambda
  \end{align}
\end{lemma}
\begin{proof} 
 We first bound $R$ by $3V_{\max}$, where we define $V_{\max}$:
\begin{align}
  V_{\max} \equiv \frac{R_{\max}}{1-\gamma} \geq R_{\max}
\end{align} We make this crude upper bound starting at the leaf node so that the probability upper bound at other subsequent steps will be bounded by the same factor. In addition, since $d_{\inft}(\mathcal{P}^{D-1}||\mathcal{Q}^{D-1})$ is bounded by $d_{\inft}^{\max}$ a.s., we can bound the resulting $t_{D-1}(\lambda,C)$ by $t_{\max}(\lambda,C)$ a.s.:
\begin{align}
  t_{D-1}(\lambda,C) &= \frac{\lambda}{3V_{\max}d_{\inft}(\mathcal{P}^{D-1}||\mathcal{Q}^{D-1})} - \frac{1}{\sqrt{C}} \geq \frac{\lambda}{3V_{\max}d_{\inft}^{\max}} - \frac{1}{\sqrt{C}} \equiv t_{\max}(\lambda,C)
\end{align}
Note that this algebra holds for all steps $d=0,\cdots, D-1$, which allows us to say $t_d(\lambda,C) \geq t_{\max}(\lambda,C)$. Thus, bounding the concentration inequality probability with $t_{\max}(\lambda,C)$ is justified when we prove \cref{lemma:step} later.\\

This probabilistic bound holds for any choice of $\{o_n\}_j$, where $\{o_n\}_j$ could be a sequence of random variables correlated with any elements of $\{s_n\}_i$. 
\begin{align}\label{eq:leafnode}
  |Q^*_{D-1}(b_{D-1},a) - \hat{Q}^*_{D-1}(\bar{b}_{D-1},a) | &\leq \lambda
\end{align}
Thus, for all $\{o_n\}_j$, $\{a_n\}$ and a fixed $a$, Eq. (\ref{eq:leafnode}) holds with probability at least $1-3\exp(-C\cdot t_{\max}^2(\lambda,C))$.
\end{proof}

\clearpage
\section{Proof of Lemma 2 (Continued)}\label{app:lemma2}
\begin{lemma}[SN Estimator Step-by-Step Convergence]
  $\hat{Q}^*_d(\bar{b}_d,a)$ is an SN estimator of $Q^*_d(b_d,a)$ for all $d = 0,\cdots,D-1$ and $a$, and the following holds with probability at least $1 - 3|A|(3|A|C)^{D}\exp(-C\cdot t^2_{\max})$:
  \begin{align}
    &|Q^*_d(b_d,a) - \hat{Q}^*_d(\bar{b}_d,a)| \leq \alpha_d\\
    &\alpha_{d} \equiv \lambda + \gamma\alpha_{d+1};\; \alpha_{D-1} = \lambda
  \end{align}
\end{lemma}
In Lemma 2, we split the difference between the SN estimator and the $Q^*$ function into two terms, the reward estimation error (A) and the next-step value estimation error (B):
\begin{align}
  |Q^*_d(b_d,a) - \hat{Q}^*_d(\bar{b}_d,a) |&\leq \undb{\left| \E[R(s_d,a)|b_d] - \frac{\sum_{i = 1}^C w_{d,i}r_{d,i}}{\sum_{i = 1}^C w_{d,i}}\right|}{(A)} + \gamma \undb{\left| \E[V_{d+1}^*(bao)|b_d] - \frac{\sum_{i = 1}^C w_{d,i}\hat{V}^*_{d+1}(\overline{b_dao_i})}{\sum_{i = 1}^C w_{d,i}} \right|}{(B)}
\end{align}
To bound these terms, we will use the SN concentration bound (\cref{thm:sn}) 3 times throughout the process.

For (A), we use the SN concentration bound to obtain the bound $\frac{R_{\max}}{3V_{\max}}\lambda$; rather than bounding $R$ with $3V_{\max}$ in this step, we instead bound $R$ with $R_{\max}$ and then augment $\lambda$ to $\frac{R_{\max}}{3V_{\max}}\lambda$ in order to obtain the same uniform $t_{\max}$ factor as the other steps. This choice of bound is made to effectively combine the $\lambda$ terms when we add (A) and (B).

For (B), we use the triangle inequality repeatedly to separate it into three terms; the importance sampling error bounded by $\lambda/3$, the Monte Carlo next-step integral approximation error bounded by $2\lambda/3\gamma$, and the function estimation error bounded by $\alpha_{d+1}$:
\begin{align}
  \begin{split}
    &(B) \leq \undb{\left| \E[V_{d+1}^*(bao)|b_d] - \frac{\sum_{i = 1}^C w_{d,i}\mathbf{V}_{d+1}^*(s_{d,i},b_d,a)}{\sum_{i = 1}^C w_{d,i}} \right|}{Importance sampling error} \\
    &+ \undb{\left| \frac{\sum_{i = 1}^C w_{d,i}\mathbf{V}_{d+1}^*(s_{d,i},b_d,a)}{\sum_{i = 1}^C w_{d,i}} - \frac{\sum_{i = 1}^C w_{d,i}V^*_{d+1}(b_dao_i)}{\sum_{i = 1}^C w_{d,i}} \right|}{MC next-step integral approximation error}\\
    &+ \undb{\left| \frac{\sum_{i = 1}^C w_{d,i}V^*_{d+1}(b_dao_i)}{\sum_{i = 1}^C w_{d,i}} - \frac{\sum_{i = 1}^C w_{d,i}\hat{V}^*_{d+1}(\overline{b_dao_i})}{\sum_{i = 1}^C w_{d,i}} \right|}{Function estimation error}\\
  \end{split}\\
  &\leq \frac{1}{3}\lambda + \frac{2}{3\gamma}\lambda + \alpha_{d+1}
\end{align}
The following subsections justify how each of the error terms are bounded. 
\subsection{Importance Sampling Error}
Before we analyze the first term, note that the conditional expectation of the optimal value function at step $d+1$ given $b_d,a$ is calculated by the following, where we introduce $\mathbf{V}_{d+1}^*(s_{d,i},b_d,a)$ as a shorthand for the next-step integration over $(s_{d+1},o)$ conditioned on $(s_{d,i},b_d,a)$:
\begin{align}
    \mathbf{V}_{d+1}^*(s_{d,i},b_d,a) &\equiv \int_S \int_O\ V_{d+1}^*(b_dao)\mathcal{Z}(o|a,s_{d+1})\mathcal{T}(s_{d+1}|s_{d,i},a)ds_{d+1}do\\
    \E[V_{d+1}^*(bao)|b_d] &= \int_S \int_S \int_O V_{d+1}^*(b_dao)(\mathcal{Z}_{d+1}) (\mathcal{T}_{d,d+1})b_d \cdot ds_{d:d+1}do \\
    &= \int_S \mathbf{V}_{d+1}^*(s_{d},b_d,a) b_d \cdot ds_{d}\\
    &= \frac{\int_{S^{d+1}} \mathbf{V}_{d+1}^*(s_{d},b_d,a)(\mathcal{Z}_{1:d}) (\mathcal{T}_{1:d}) b_{0} ds_{0:d}}{\int_{S^{d+1}} (\mathcal{Z}_{1:d}) (\mathcal{T}_{1:d}) b_{0} ds_{0:d}}
\end{align}
Noting that the first term is then the difference between the SN estimator and the conditional expectation, and that $||\mathbf{V}_{d+1}^*||_{\inft} \leq V_{\max}$, we can apply the SN inequality for the second time in Lemma 2 to bound it by the augmented $\lambda/3$. 

\subsection{Monte Carlo Next-Step Integral Approximation Error}
The second term can be thought of as Monte Carlo next-step integral approximation error. To estimate $\mathbf{V}_{d+1}^*(s_{d,i},b_d,a)$, we can simply use the quantity $V^*_{d+1}(b_dao_i)$, as the random vector $(s_{d+1,i},o_i)$ is jointly generated using $G$ according to the correct probability $\mathcal{Z}(o|a,s_{d+1})\mathcal{T}(s_{d+1}|s_{d,i},a)$ given $s_{d,i}$ in the POWSS simulation. Consequently, the quantity $V^*_{d+1}(b_dao_i)$ for a given $(s_{d,i},b_d,a)$ is an unbiased 1-sample MC estimate of $\mathbf{V}_{d+1}^*(s_{d,i},b_d,a)$. We define the difference between these two quantities as $\Delta_{d+1}$, which is implicitly a function of random variables $(s_{d+1,i},o_i)$:
\begin{align}
  \Delta_{d+1}(s_{d,i},b_d,a) &\equiv \mathbf{V}_{d+1}^*(s_{d,i},b_d,a) - V^*_{d+1}(b_dao_i)
\end{align}
Then, we note that $||\Delta_{d+1}||_{\inft} \leq 2V_{\max}$ and $\E\Delta_{d+1} = 0$ by the Tower property conditioning on $(s_{d,i},b_d,a)$ and integrating over $(s_{d+1,i},o_i)$ first, which holds for any choice of well-behaved sampling distributions on $\{s_{0:d}\}_i$. Using this fact, we can then consider the second term as an SN estimator for the bias $\E\Delta_{d+1} = 0$, and use our SN concentration bound for the third time. Since $||\Delta_{d+1}||_{\inft} \leq 2V_{\max}$, our $\lambda$ factor is then augmented by 2/3:
\begin{align}
\begin{split}
  &\left| \frac{\sum_{i = 1}^C w_{d,i}\mathbf{V}_{d+1}^*(s_{d,i},b_d,a)}{\sum_{i = 1}^C w_{d,i}} - \frac{\sum_{i = 1}^C w_{d,i}V^*_{d+1}(b_dao_i)}{\sum_{i = 1}^C w_{d,i}} \right| \\
  &= \left| \frac{\sum_{i = 1}^C w_{d,i}\Delta_{d+1}(s_{d,i},b_d,a)}{\sum_{i = 1}^C w_{d,i}} - 0\right| \leq \frac{2}{3}\lambda \leq \frac{2}{3\gamma}\lambda
\end{split}
\end{align}

\subsection{Function Estimation Error}
Lastly, the third term is bounded by the inductive hypothesis, since each $i$-th absolute difference of the $Q$-function and its estimate at step $d+1$, and furthermore the value function and its estimate at step $d+1$, are all bounded by $\alpha_{d+1}$. 

\clearpage
\stepcounter{theorem}
\section{Proof of Theorem 3}\label{app:theorem3}
\subsection{Belief State Policy Convergence Lemma}

Before we prove Theorem \ref{thm:powss-grand}, we first prove the following lemma, which is an adaptation of Kearns \textit{et al.}~\shortcite{Kearns2002} and Singh and Yee~\shortcite{Singh1994} for belief states $b$.
\stepcounter{lemma}
\begin{lemma}\label{lemma:value}
    Suppose we obtain a greedy policy implemented by some approximation $\tilde{V}_{d,t}(b) \approx V_{t+d}^*(b)$ by generating a tree at online step $t$ and obtaining a value function estimate at tree depth $d$ for a belief $b$. Define the total loss $L_{\tilde{V},t}\equiv V_{t}^*(b) - V_{\tilde{V}_{0,t}}(b)$ as the difference between the value obtained by the optimal policy and the value obtained by the $\tilde{V}_{d,t}$ approximation at online step $t$. If $|\tilde{V}_{d,t}(b) - V_{t+d}^*(b)|\leq \beta$ for all online steps $t \in [0,D-1]$ and its corresponding tree depth $d=0,\cdots,D-1-t$, then the total loss by implementing the greedy policy from the beginning is bounded by the following:
  \begin{align}
    L_{\tilde{V},0}(b) &\leq \frac{3}{1-\gamma} \beta
  \end{align}
\end{lemma}

\begin{proof} We mirror the proof strategies given in Kearns \textit{et al.}~\shortcite{Kearns2002} and Singh and Yee~\shortcite{Singh1994} for belief states $b$. 

Consider the optimal action $a = \pi_d^*(b)$ and the greedy action $\tilde{a}=\pi_{\tilde{V},t}(b)$. Here, we denote $R(b,a)$ as the shorthand notation for $\E[R(s,a)|b]$. Since $\tilde{a}$ is greedy, it must look at least as good as $a$ under $\tilde{V}$:
\begin{align}
  R(b,a) + \gamma \E[\tilde{V}_{1,t}(bao)|b] \leq R(b,\tilde{a}) + \gamma \E[\tilde{V}_{1,t}(b\tilde{a}o)|b]
\end{align}
Since we have $|\tilde{V}_{d,t}(b) - V_{t+d}^*(b)|\leq \beta$,
\begin{align}
  R(b,a) + \gamma \E[V_{t+1}^*(bao)-\beta|b] &\leq R(b,\tilde{a}) + \gamma \E[V_{t+1}^*(b\tilde{a}o)+\beta|b]\\
  R(b,a)-R(b,\tilde{a})&\leq 2\gamma\beta + \gamma\E[ V_{t+1}^*(b\tilde{a}o)|b] - \gamma\E[V_{t+1}^*(bao)|b]
\end{align}
Then, the loss for $b$ at time $t$ is:
\begin{align}
  L_{\tilde{V},t}(b) &= V_{t}^*(b) - V_{\tilde{V}_{0,t}}(b) \\
  &= R(b,a)-R(b,\tilde{a})+\gamma\E[ V_{t+1}^*(bao)|b] - \gamma\E[V_{\tilde{V}_{0,t+1}}(b\tilde{a}o)|b]
\end{align}
Substituting the reward function into the loss expression,
\begin{align}
  L_{\tilde{V},t}(b) &= R(b,a)-R(b,\tilde{a})+\gamma\E[ V_{t+1}^*(bao)|b] - \gamma\E[V_{\tilde{V}_{0,t+1}}(b\tilde{a}o)|b]\\
  &\leq 2\gamma\beta + \gamma\E[ V_{t+1}^*(b\tilde{a}o)|b] - \gamma\E[V_{t+1}^*(bao)|b]+\gamma\E[ V_{t+1}^*(bao)|b] - \gamma\E[V_{\tilde{V}_{0,t+1}}(b\tilde{a}o)|b]\\
  &\leq 2\gamma\beta + \gamma\E[ V_{t+1}^*(b\tilde{a}o)|b] - \gamma\E[V_{\tilde{V}_{0,t+1}}(b\tilde{a}o)|b]\\
  &\leq 2\gamma\beta + \gamma\E[ L_{\tilde{V},t+1}(b\tilde{a}o)|b]
\end{align}
Note that we have $L_{\tilde{V},D-1}(b) \leq \beta$ from the root node estimate at the last step, which means we obtain the bound with some over-approximations:
\begin{align}
  L_{\tilde{V},0}(b) &\leq \sum_{d=1}^{D-1} 2\beta \gamma^d + \gamma^{D-1} \beta \leq \sum_{d=0}^{D-1} 3\beta \gamma^d \leq \frac{3}{1-\gamma} \beta
\end{align}
These over-approximations are done in order to generate constants that can be easily calculated.
\end{proof}

\subsection{Proof of Theorem 3}
We reiterate the conditions and \cref{thm:powss-grand} below:
\begin{enumerate}[(i)]
   \item $S$ and $O$ are continuous spaces, and the action space has a finite number of elements, $|A|<+\inft$.
    \item For any observation sequence $\{o_n\}_j$, the densities $\mathcal{Z},\mathcal{T},b_0$ are chosen such that the R\'enyi divergence of the target distribution $\mathcal{P}^d$ and sampling distribution $\mathcal{Q}^d$ (\cref{eq:pqdef,eq:pqdef2}) is bounded above by $d_{\inft}^{\max} < +\inft$ a.s. for all $d = 0,\cdots,D-1$: 
      $$d_{\inft}(\mathcal{P}^d||\mathcal{Q}^d) = \text{ess sup}_{x \sim \mathcal{Q}^d}w_{\mathcal{P}^d/\mathcal{Q}^d}(x) \leq d_{\inft}^{\max}$$
    \item The reward function $R$ is Borel and bounded by a finite constant $||R||_{\inft} \leq R_{\max} < +\inft$ a.s., and $V_{\max} \equiv \frac{R_{\max}}{1-\gamma}<+\infty$.
    \item We can evaluate the generating function $G$ as well as the observation probability density $\mathcal{Z}$.
    \item The POMDP terminates after $D < \infty$ steps. 
\end{enumerate}
\begin{theorem}[POWSS Policy Convergence] 
  In addition to conditions (\ref{req:space})-(\ref{req:finite}), assume that the closed-loop POMDP Bayesian belief update step is exact.
  Then, for any $\epsilon > 0$, we can choose a $C$ such that the value obtained by POWSS is within $\epsilon$ of the optimal value function at $b_0$ a.s.:
  \begin{align}
    V^*(b_0)-V^{\textsc{POWSS}}(b_0) &\leq \epsilon
  \end{align}
\end{theorem}

\begin{proof} In our main report, we have proved \cref{lemma:leaf,lemma:step}, which gets us the root node convergence for all actions. We apply these lemmas as well as \cref{lemma:value} to prove the policy convergence.

From \cref{lemma:step}, we have that the error in estimating $Q^*$ with our POWSS policy is bounded by $\lambda/(1-\gamma)$ for all $d,a$ with probability at least $1-\delta$. This directly implies that the $V$-function estimation errors are bounded as well for all steps $d$; if $|Q_d^*(b_d,a)-\hat{Q}_d^*(\bar{b}_d,a)|\leq \frac{\lambda}{1-\gamma}$ for all $d,a,$ then:
\begin{align}
  |\max_{a\in A}Q_d^*(b_d,a)-\max_{a\in A}\hat{Q}_d^*(\bar{b}_d,a)| &= |V_d^*(b_d)-\hat{V}_d^*(\bar{b}_d)| \leq \frac{\lambda}{1-\gamma}
\end{align}

For online planning, we require that the POWSS trees generated at each online planning step must satisfy the concentration inequalities for all of its nodes. As each of the trees generated for $D$ steps need to have good estimates, we worst-case upper bound the union bound probability by multiplying $D$ to $\delta$. Applying \cref{lemma:value}, we get that if all the nodes satisfy the concentration inequality, which happens with probability at least $1-D\delta$, the following holds:
\begin{align}
  V^*(b_0)-V^{\textsc{POWSS}}(b_0) &= L_{\hat{V}^*,0}(b_0) \leq \frac{3\lambda}{(1-\gamma)^2}
\end{align}
Note that the maximum difference between the values obtained by the two policies is bounded by $2V_{\max}$. At each online step, you can have $2R_{\max}$ as the maximum possible difference between the two rewards the agent can obtain via the greedy POWSS policy and the optimal policy generated at each online planning step, and at each online step there exists a discount $\gamma$. We can use this bound for the bad case probability $D\delta$. Using all the definitions of the constants defined in \cref{thm:powss}:
\begin{align}
  V^*(b_0)-V^{\textsc{POWSS}}(b_0) &= \E\lrs{\sum_{i=0}^{D-1}\gamma^i R(s_i,\pi_i^*(s_i))\given b_0} - \E\lrs{\sum_{i=0}^{D-1}\gamma^i R(s_i,\pi_i^{\textsc{POWSS}}(s_i))\given b_0}\\
  &\leq (1-D\delta)\frac{3\lambda}{(1-\gamma)^2} + D\delta\sup\lrc{\sum_{i=0}^{D-1}\gamma^i |R(s_i,\pi_i^*(s_i))-R(s_i,\pi_i^{\textsc{POWSS}}(s_i))|\given b_0}\\
  &= (1-D\delta)\frac{3\lambda}{(1-\gamma)^2} + D\delta\sum_{i=0}^{D-1}\gamma^i (2R_{\max})\\
  &\leq (1-D\delta)\frac{3\lambda}{(1-\gamma)^2} + D\delta\frac{2R_{\max}}{1-\gamma}\\
  &\leq \frac{3\lambda}{(1-\gamma)^2} + 2D\delta V_{\max} =\frac{5\lambda}{(1-\gamma)^2}= \epsilon
\end{align}
Therefore, we obtain our desired bound on the values obtained by POWSS policy.
\end{proof}

\end{document}